\documentclass{article}

\usepackage{PRIMEarxiv}

\usepackage[utf8]{inputenc} 
\usepackage[T1]{fontenc}    
\usepackage{hyperref}       
\usepackage{url}            
\usepackage{booktabs}       
\usepackage{amsfonts}       
\usepackage{nicefrac}       
\usepackage{microtype}      
\usepackage{lipsum}
\usepackage{fancyhdr}       
\usepackage{graphicx}       
\graphicspath{{media/}}     
\usepackage{amsmath}
\usepackage{color}
\usepackage{enumitem}
\usepackage{hyperref}
\usepackage{subfigure}
\usepackage{amsthm}
\usepackage{amssymb}
\usepackage{lipsum}
\usepackage{listings}
\newtheorem{definition}{Definition}
\newtheorem{theorem}{Theorem}
\newtheorem{prop}{Proposition}

\newtheorem*{preuve*}{Proof}
\usepackage{appendix}
\usepackage{cite}
\usepackage{algorithm}
\usepackage{algpseudocode}
\usepackage{diagbox}
\usepackage{footnote}

\title{Adaptive Time Series Forecasting with Markovian Variance Switching
}

\author{
  Baptiste Abélès  \\
  Universidad Pompeu Fabra \\
  \texttt{baptistabeles@gmail.com}
   \And Joseph de Vilmarest \\
    Viking Conseil \\
  \texttt{joseph.de-vilmarest@vikingconseil.fr} \\
   \And 
  Olivier Wintemberger \\
  Sorbonne Université \\
  \texttt{olivier.wintenberger@sorbonne-universite.fr} \\
}
\begin{document}

\newcommand{\E}{\mathbb{E}}
\newcommand{\N}{\mathcal{N}}
\newcommand{\Z}{\mathbb{Z}}
\newcommand{\Q}{\mathbb{Q}}
\newcommand{\R}{\mathbb{R}}
\newcommand{\C}{\mathbb{C}}
\newcommand{\D}{\mathbb{D}}
\newcommand{\K}{\mathbb{K}}
\newcommand{\F}{\mathcal{F}}
\newcommand{\T}{\mathbb{T}}
\maketitle

\begin{abstract}
Adaptive time series forecasting is essential for prediction under regime changes.
Several classical methods assume linear Gaussian state space model (LGSSM) with variances constant in time. However, there are many real-world processes that cannot be captured by such models. We consider a state-space model with Markov switching variances.
Such dynamical systems are usually intractable because of their computational complexity increasing exponentially with time; Variational Bayes (VB) techniques have been applied to this problem. In this paper, we propose a new way of estimating variances based on online learning theory; we adapt expert aggregation methods to learn the variances over time. We apply the proposed method to synthetic data and to the problem of electricity load forecasting.
We show that this method is robust to misspecification and outperforms traditional expert aggregation.
\end{abstract}


\section{Introduction}
\label{sec::intro}
Time series forecasting is a fundamental issue. A few examples are weather forecasts for farmers, sales and inventory forecasting in retail, price and cost forecasting for industrials.
In this paper, a special motivation is electricity load forecasting, a crucial task for grid operators as the production must balance the consumption in real time \cite{bunn1985comparative,hong2014global}.

State-space models have been widely employed to model the temporal behavior of data. In particular, linear Gaussian state-space models (LGSSMs) yield tractable posterior distributions obtained by the Kalman filter (KF) \cite{kalman1960new}. Many situations are filled with non linearity both in dynamics and observation equation. A vast literature has been devoted to deal with non-linear dynamical systems, e.g. the extended Kalman filter (EKF) \cite{Dur} which essentially consists to linearize the model using Taylor first-order expansion, and the unscentend Kalman filter (UKF) \cite{wan2000unscented} which is radically different and consists to apply the unscented transformation to both prediction and filtering step.

Although the inference in LGSSMs is known \cite{kalman1960new}, the Kalman filter crucially depends on the variances of the observation and state noises. The choice of these variances may be seen as the parametrization of a second-order gradient descent algorithm, once we remark the equivalence of the Kalman recursions and gradient updates \cite{ollivier2018online, de2021stochastic}. The most widely-employed setting of these variances is time-invariant \cite{fahrmeir1994multivariate,Dur,davis2016introduction}, and the natural choice of these variances is maximum-likelihood. However, constant variances mean a smooth evolution of the environment, while in some cases there are changes of regime. To tackle that problem, various methods have been introduced, often under the paradigm adaptive Kalman filter \cite{mehra1970identification, huang2017novel, huang2020slide}.

Alternatively the switching Kalman filter (SKF) \cite{murphy1998switching} considers several linear dynamic models and try to combine them with a switching mechanism (such as a Markov chain), selecting over time one among several regimes. This latter approach is of a particular interest in many real-world problems that cannot be approximated with a single regime, and where multiple behaviors are needed instead. For instance, time series forecasting during a period covering stable as well as unstable moments motivates switches between different regimes. Our application to electricity load forecasting is one of these, as we include very unstable periods such as lockdowns due to the coronavirus in 2020, during which it was best to switch to a more reactive method \cite{obst2021adaptive}. 

The complexity of these hybrid dynamical systems, \textit{i.e}. systems that combine discrete and continuous valued latent variables, increases exponentially with time. State-of-the art approaches propose to learn a switching linear dynamical system using Variational Bayesian (VB) techniques \cite{alameda2021variational}. However, VB approaches have some limits under model misspecification \cite{wang2019variational}. 

In this paper, we propose to consider the SKF as an aggregation of experts, where each expert corresponds to a different regime. We apply online learning techniques to track the best expert as in \cite{herbster1998tracking, cesa2006prediction}. After introducing various methods in Section \ref{sec::Relatew_work}, we define our algorithm in Section \ref{sec::methodology}, and we provide a regret bound in expectation. Then, in Section \ref{sec:experiments}, we conduct various experiments on synthetic data, as well as a real data set, forecasting the French national electricity load. We study well-specified and misspecified experiments and we show that this new framework is robust to misspecification and outperforms Kalman aggregation as in \cite{adjakossa2023kalman}.

\section{Background}
\label{sec::Relatew_work}
\paragraph{Notations} Given probability measures $p$ and $q$ with a common dominating measure $\mu$, we denote the Kullback-Leibler divergence as $D(p||q) := \int p \log(p/q)d\mu$. For any positive integer $t$ we denote $[t] := \{1,...,t\}$ and $\F_t := \sigma(\boldsymbol{x}_1,y_1,...,\boldsymbol{x}_t,y_t,\boldsymbol {x}_{t+1})$ is the sigma algebra generated by the pair $(\boldsymbol{x}_i,y_i)_{i=1}^t$ and $\boldsymbol{x}_{t+1}$. For some $\mu,\sigma \in \R$, the notation $\N(.|\mu,\sigma)$ refers to as the normal distribution with mean $\mu$ and variance $\sigma^2$ (we keep the same notation if $\boldsymbol{\mu}$ is a vector and $\Sigma$ is the covariance matrix associated). We denote $\Delta(\alpha,\beta,\gamma)$ as the diagonal matrix in $\R^{3\times3}$ with coefficients $\alpha,\beta,\gamma$.

Let $y_t \in \R$ be the variable of interest, $\boldsymbol{x}_t \in \R^d$ a vector of explainable variables and $\boldsymbol{\theta}_t \in \R^d$ the vector of latent variables at time $t$. The classical LGSSM writes :
\begin{align}
\label{LGSSM::var::eq_obs}
    & y_t = \boldsymbol{\theta}_ t^T\boldsymbol{x}_t + \epsilon_t, \quad \text{where } \epsilon_t \sim \mathcal{N}(0,\sigma_{t}^2), \\
\label{LGSSM::var::eq_state}& \boldsymbol{\theta}_{t+1} = \boldsymbol{\theta}_t + \nu_t, \quad \text{where } \nu_t \sim \mathcal{N}(0,Q_{t}),
\end{align}
Our goal is to estimate jointly the state and variances knowing the past that we denote $\boldsymbol{\hat{\theta}}_{t|t-1} = \E[\boldsymbol{\theta}_t|\F_{t-1}]$ and $P_{t|t-1} =\E[(\boldsymbol{\theta}_t-\boldsymbol{\hat{\theta}}_{t|t-1})(\boldsymbol{\theta}_t-\boldsymbol{\hat{\theta}}_{t|t-1})^T|\F_{t-1}]$. The well-known KF gives the exact updates if $\sigma_{t}^2$ and $Q_{t}$ are known :
  \begin{theorem}(Kalman Filter)
  \label{Kalman_update}
\begin{align*}
     &P_{t|t}= P_{t|t-1} -\frac{P_{t|t-1}\boldsymbol{x}_t\boldsymbol{x}_t^T P_{t|t-1}}{\boldsymbol{x}_t^T P_{t|t-1}\boldsymbol{x}_t + \sigma_{t}^2},  \\
   & \boldsymbol{\hat{\theta}}_{t|t} = \boldsymbol{\hat{\theta}}_{t|t-1} -\frac{P_{t|t}}{\sigma_{t}^2}\left(\boldsymbol{x}_t(\boldsymbol{\hat{\theta}}_{t|t-1}^T\boldsymbol{x}_t - y_t)\right),  \\
     & P_{t+1|t} = P_{t|t} + Q_{t}, \\ &\boldsymbol{\hat{\theta}}_{t+1|t} = \boldsymbol{\hat{\theta}}_{t|t}. 
\end{align*}
\end{theorem}
Usual choice in the literature is to assume that the variances are time-invariant \cite{Dur, davis2016introduction} and apply the Expectation-Maximisation (EM) algorithm. A second paradigm which has frequently been called Adaptive Kalman Filter (AKF) \cite{mehra1970identification} consists in estimating variances over time and applying the Kalman recursive updates using such adaptive variances. 

A more recent approach consists in finding the complete posterior of both state and variances \cite{de2021viking}.
The authors assume that $Q_t, \sigma_t^2$ are random variables and try to estimate the joint distribution of the state and the variances in a dynamical way. They assume a model on the dynamics of the variances; then, starting from a prior on $(\boldsymbol{\theta}_1,\sigma_1^2,Q^{(1)})$, they propose to recursively estimate the posterior distribution of $(\boldsymbol{\theta}_t,\sigma_t^2,Q_t)$.
However there is no natural parametric class of distributions on the joint distribution such that the posterior remains in the class considered. To handle this issue, they apply the Variational Bayesian (VB) approach \cite{vsmidl2006variational}; it consists in approximating the posterior with the best factorized form in the sense of the minimal Kullback-Leibler divergence.
We believe that the limits of such methods lie on too much degree of freedom and we aim to constrain the model.

To reduce the degree of freedom of the problem, one can assume the existence of a hidden Markov chain $(Z_t) \in [K]$ parameterizing the dynamics of the classical LGSSM. At each time $t$, $\sigma_t^2= \sigma_{Z_t}^2, Q_t = Q_{Z_t}$.
This model writes :
\begin{align}
\label{SKM::var::eq_obs}
    & y_t = \boldsymbol{\theta}_ t^T\boldsymbol{x}_t + \epsilon_t, \quad \text{where } \epsilon_t \sim \mathcal{N}(0,\sigma_{Z_t}^2), \\
    & \label{SKM::var::eq_state}\boldsymbol{\theta}_{t+1} = \boldsymbol{\theta}_t + \nu_t, \quad \text{where } \nu_t \sim \mathcal{N}(0,Q_{Z_t}), \\
    & \label{var::eq_markov}
    \mathbb{P}(Z_{t+1} = k) = \sum_{i=1}^KM_{ik} \mathbb{P}(Z_t =i),
\end{align}
where $M_{ik} = \mathbb{P}(Z_{t+1}=k|Z_t=i)$ for $i,k \in [K]$ is an entry of the transition matrix $\boldsymbol{M} \in \R^{K\times K}$.

However, the complexity of the posterior distribution of such dynamical systems grows exponentially in time, making it intractable. Indeed the posterior writes :
\begin{align*}
\label{posterior}
p(\boldsymbol{\theta}_{1:t},Z_{1:t}|\F_t)&\propto \pi(Z_1) \mathcal{N}(\boldsymbol{\theta}_1|\boldsymbol{\hat{\theta}}_1,P_1) \prod_{s=1}^t \mathcal{N}(y_s|\boldsymbol{\theta}_s^T\boldsymbol{x}_s,\sigma_{Z_s} ^2)\\
    & \prod_{s=1}^{t-1} \mathcal{N}(\boldsymbol{\theta}_{s+1}|\boldsymbol{\theta}_s,Q_{Z_{s}})M_{Z_{s+1},Z_s},
\end{align*}
where $\pi(Z_1)$ denotes the prior law on $Z_1$.
Several approaches, referred to as \textit{Generalized Pseudo Bayesian algorithm} \cite{murphy1998switching}, approximate this exponential mixture of Gaussians with a smaller mixture of Gaussians. 
Recent works use VB techniques to approximate the posterior \cite{alameda2021variational}. The main idea is to look for a tractable approximation of the posterior. If we denote $p_{\F_t}$ the posterior we look for $(q_Z^*,q_\theta^*) \in \mathcal{D}_t$ such that :
$$(q_Z^*,q_\theta^*) \in \arg \min_{(q_Z,q_\theta) \in \mathcal{D}_t} D(q_Z \times q_\theta \  || \ p_{\F_t}),$$
where $\mathcal{D}_t$ is the parametric class considered. We provide the derivation and results of this method that performs poorly in our experiments in supplementary materials. An important limitation of this method
 is the lack of robustness to misspecification \cite{wang2019variational}. 

To overcome this problem we propose a new approach to adaptive variance estimation relying on expert aggregation. Aggregation of experts is a well-established framework in machine learning \cite{cesa2006prediction,vovk1995game}. In particular, it yields excellent results in the context of electricity load forecasting \cite{adjakossa2023kalman}. Our work follows this classical setting, we consider a learner, $K$ experts and a loss function $\ell : \boldsymbol{\mathcal{E}} \times \mathcal{Y} \to \R$ where $\boldsymbol{\mathcal{E}}$ is the space of experts predictions and $\mathcal{Y}$ is the space of observation. At each time $t$, each expert $k \in [K]$ predicts $\boldsymbol{e}_{k,t} \in \boldsymbol{\mathcal{E}}$. The learner picks a distribution $\boldsymbol{p}_t$ over the set of experts' predictions, predicts $\boldsymbol{e}_t=\sum_{k=1}^K\boldsymbol{p}_{k,t}\boldsymbol{e}_{k,t} \in \boldsymbol{\mathcal{E}}$ and observes his loss $\ell_t := \ell(\boldsymbol{e}_t,y_t)$. The learner has also access to the loss of each expert at time $t$, that we denote $\ell_{k,t}:=\ell(\boldsymbol{e}_{k,t},y_t)$. The goal of the learner is to minimize his cumulative loss on sequence of length $T$. A first way to measure his performance is to compare his cumulative loss with other expert during the sequence. 
\begin{definition}
    We define the regret with respect to the expert $k$ for a sequence of length $T$ as : 
$$R_T(k) = \sum_{t=1}^T \ell_t-\ell_{k,t}$$
\end{definition}
\subsection{Exponential-Weight-Average (EWA) algorithm for Kalman Filter}
\label{subsec :: EWA}
To better understand our framework, we present the famous EWA to show how a learner can optimize his strategy to control his regret. 

\textit{We make a small remark here that is we will essentially focus on the variance of the state Equation \ref{SKM::var::eq_state}. Indeed Kalman Filter gives exactly the same update for the pair $(\sigma^2,Q)$ and $(1,Q^*)$ with $Q^* = \frac{Q}{\sigma^2}$ (refer to \cite{ba2016line}).}

Let's consider the notations of the state space model and assume each expert is a KF derived with $Q_t^*=Q^{*(k)}$ which starting from some $\boldsymbol{\hat\theta}_{0|0}$ yields the prediction $\boldsymbol{\hat{\theta}}_{t|t-1}^{(k)}$ of the state at time $t$, for $t \in [T]$. In this case, let's identify  the space of experts as the observation space , \textit{i.e} $\boldsymbol{\mathcal{E}} = \mathcal{Y}$ and each expert predicts $\boldsymbol{e}_{k,t}=\hat{y}_{k,t}$ where $\hat{y}_{k,t}=\boldsymbol{\hat{\theta}}_{t|t-1}^{(k)T}\boldsymbol{x}_t$.
We consider the quadratic loss $\ell_{k,t}=(y_t-\hat{y}_{k,t})^2$.

\begin{algorithm}
\caption{EWA Algorithm}\label{alg:EWA}
\begin{algorithmic}
\State \textbf{Set} $\boldsymbol{p}_1 = K^{-1} \mathbf{1}_{K}$ 
\For{$t\in [T]$}
\State select $\boldsymbol{p}_t$, forecast $\hat{y}_t= \sum_{k=1}^K \boldsymbol{p}_{t}(k)\hat{y}_{k,t}$
\State derive loss $\ell_t$ and $\ell_{k,t}$ for $k \in [K]$
\State for $k \in [K] $,update:
$\boldsymbol{p}_{t+1}(k) = \frac{\boldsymbol{p}_t(k)e^{-\eta \ell_{k,t}}}{\sum_{j=1}^K \boldsymbol{p}_t(j)e^{-\eta \ell_{j,t}}}$
\EndFor
\end{algorithmic}
\end{algorithm}
\newpage

A first important result is given by the following proposition. 
\begin{prop}
   Assume that the loss function $\ell$ is convex and  takes value in $[0,1]$:
   $$R_T^* \leq \frac{\log(K)}{\eta} + \frac{\eta T}{8}\,.$$
Refer to \cite{cesa2006prediction}, Theorem 2.2 for a proof. 
\end{prop}
Algorithm \ref{alg:EWA} belongs to the so-called class of \textit{Static Expert algorithm} which tend  to converge to the best expert. However in our case since we assume switches; we look for aggregation algorithms able to switch from one expert to the other.

\subsection{Fixed-Share (FS) algorithm}
For this reason we should consider a more refined measure of performance.
\begin{definition}
    Let $k^T = (k_1,...,k_T)$ be a sequence of experts we can define the regret with respect to this sequence as:
    $$R_T(k^T) = \sum_{t=1}^T \ell_t - \ell_{k_t,t}$$
    and $$R_T^* = \sup_{k^T} R_T(k^T)$$
\end{definition}
A first and naïve approach would be to consider each sequence as a new expert and try to recover the best expert on the set of every possible sequence of expert using EWA algorithm starting from a prior distribution $p_1(k^T)$ over this set. However in practice we will face a problem of intractability since we need to store and update $O(K^T)$ weights to compare to $R_T(k^{T})$.   
The problem of controlling the regret with respect to sequences of experts, known as tracking the best expert, was introduced by Herbster and Warmuth \cite{herbster1998tracking}, who proposed the simple FS algorithm with good regret guarantees.
A key fact highlighted in  \cite{vovk1997derandomizing} (section 3.3) is that FS, can be interpreted as the EWA algorithm on sequences of experts under a markov prior. 

Below we present a generalization of FS named as Markov-Hedge (MH) algorithm which starting from markov chain prior $\boldsymbol{p}_1(k_1,...,k_T)=\pi(k_1)M_{k_2|k_1}...M_{k_T|k_{T-1}}$ collapses to efficient algorithm with bounded regret \cite{mourtada2017efficient}. For $t\in[T]$, $M_{k_t|k_{t-1}}$ correspond to the entry of the transition matrix $\boldsymbol{M} \in \R^{K \times K}$ of the markov chain.

 \begin{algorithm}[!h]
\caption{Markov-Hedge Algorithm}\label{alg_MH}
\begin{algorithmic}
\State \textbf{Require:} $\boldsymbol{M} = (M_{k_t|k_{t-1}}), \boldsymbol{p}_1 = K^{-1} \mathbf{1}_{K}$ 
\For{$t\in [T]$}
\State select $\boldsymbol{p}_t$, forecast $\hat{y}_t=\sum_{k=1}^K p_{k,t}\hat{y}_{k,t}$
\State derive loss $\ell_t$ and $\ell_{k,t}$ for $k \in [K]$
\State for $k \in [K] $,update:
$\boldsymbol{p}_{t+1}(k) =  \frac{\boldsymbol{p}_t(k)e^{-\eta \ell_{k,t}}}{\sum_{j=1}^K \boldsymbol{p}_t(j)e^{-\eta \ell_{j,t}}}$
\State update: $\boldsymbol{p}_{t+1} = M\boldsymbol{p}_{t+1}$ \textit{i.e}
for $k \in [K]$: $$\boldsymbol{p}_{t+1}(k)= \sum_{j=1}^KM_{k|j} \boldsymbol{p}_{t+1}(j)$$
\EndFor
\end{algorithmic}
\end{algorithm}
\newpage 
  We have the following proposition : 
\begin{prop}
\label{prop::FS}
Algorithm Markov-Hedge, with initial distribution $\boldsymbol{p}_1(k_1,...,k_T)=\pi(k_1)M_{k_2|k_1}...M_{k_T|k_{T-1}}$ and transition matrix $\boldsymbol{M} \in \R^{K \times K}$, guarantees the following regret bound. For $T \ge 1$ and any sequence  $k^T = (k_1,...,k_T)$ : 
   $$ R_T(k^T) \le \frac{1}{\eta}\log\frac{1}{\pi(k_1)} +\frac{1}{\eta}\sum_{t=2}^T \log \frac{1}{M_{k_t|k_{t-1}}}+ \frac{\eta T}{8}\,.$$
Refer to \cite{cesa2006prediction} Theorem 5.1 for a proof
\end{prop}

\section{Methodology}
\label{sec::methodology}
\textit{Following the remark made in Subsection \ref{subsec :: EWA}
we will first consider $\sigma_t^2=1$ for $t \in [T]$.}
\subsection{Well-specified setting}
As we saw in Section \ref{sec::Relatew_work} what is usually done \cite{adjakossa2023kalman} is to perform an aggregation over the space of observations. This means that we consider our $K$ experts as $K$ Kalman Filters derived with constant variances $Q_t=Q^{(k)},\sigma_t^2=1$. The expert $k$ predicts $\hat{y}_{k,t}$ for all $t \in [T]$ and updates his weights using Algorithm \ref{alg_MH}. The main contribution of this paper is to apply the aggregation framework over the space of variances, as displayed by diagrams in Figure \ref{fig:schema}.
\begin{figure}[!h]
    \centering
    \includegraphics[scale=0.5]{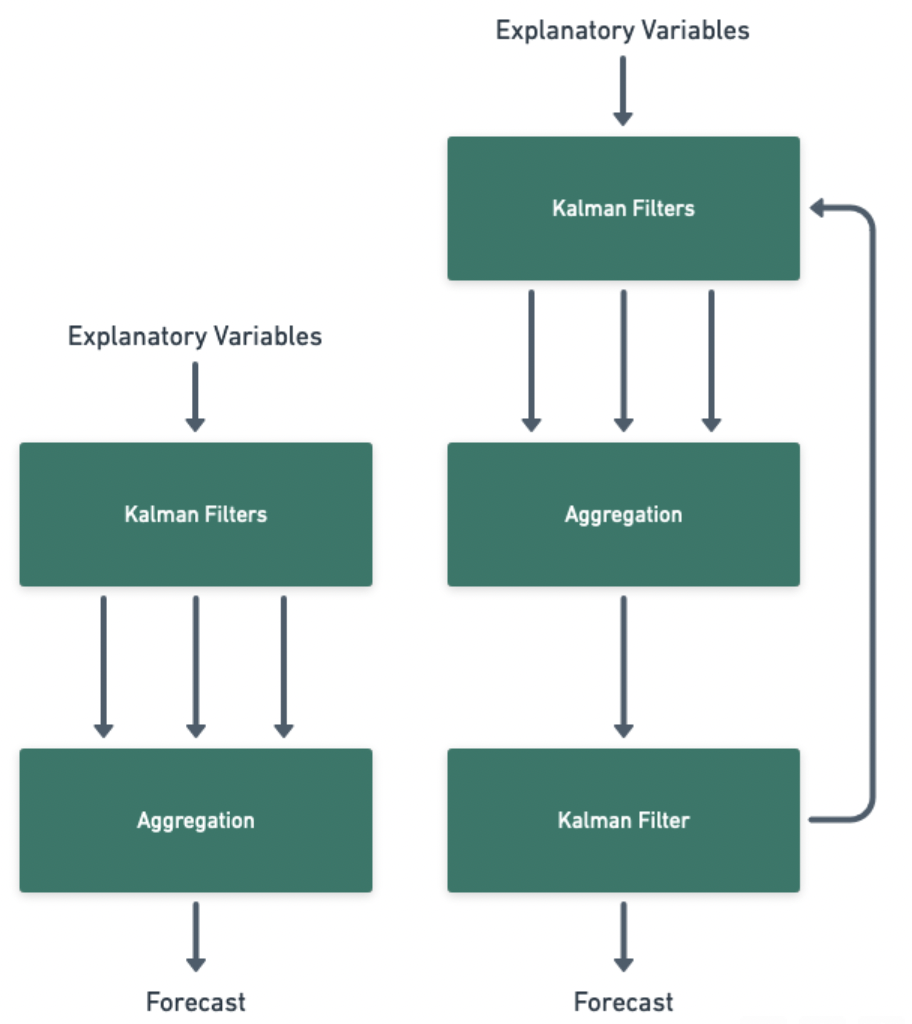}
    \caption{Aggregation of Kalman (left) versus Kalman with aggregation of variances (right).}
    \label{fig:schema}
\end{figure}
The algorithm that we design is a Kalman filter that competes the best covariance matrix at each time $t$ using expert aggregation on $\boldsymbol{\mathcal{Q}}=\{Q^{(1)},...,Q^{(K)}\}$. We consider the space of experts $\boldsymbol{\mathcal{E}}=\boldsymbol{\mathcal{Q}}$ so that at time $t$, each expert $k$ computes $P_{t|t-1}^{(k)},\boldsymbol{\hat\theta}_{t|t-1}^{(k)}$ using equations of Theorem \ref{Kalman_update} and predicts:
$$\hat{y}_{k,t}=\boldsymbol{\hat\theta}_{t|t-1}^{(k)T}\boldsymbol{x}_t.$$
The learner observes the loss $\ell_{k,t}=(y_t-\hat{y}_{k,t})^2$ of each expert and updates his weights as in Algorithm \ref{alg_MH}. Then he selects $ Q^{(k)}$ with probability $\boldsymbol{p}_t(k)$ to compute $P_{t+1|t},\boldsymbol{\theta}_{t+1|t}$ given his previous predictions $P_{t|t-1},\boldsymbol{\theta}_{t|t-1}$ and predicts: $$\hat{y}_{t+1}=\boldsymbol{\hat\theta}_{t+1|t}^T\boldsymbol{x}_{t+1}.$$

This yields the randomized version of our algorithm, Kalman Filter Markov-Hedge (KFMH), detailed in Algorithm \ref{alg:KFMH_random}.

\begin{algorithm}
\caption{Randomized KFMH}\label{alg:KFMH_random}
\begin{algorithmic}
\State \textbf{Require:} $\boldsymbol{p}_1 = K^{-1}\mathbf{1}_K,  \boldsymbol{M} =(M_{i,j})_{i,j=1}^{  K}$
\For{$t\in [T]$}
\State Reveal $y_t$ and ${\boldsymbol x}_{t-1}$
\For{$k \in [K]$}
\State compute $P_{t|t-1}^{(k)}$, $\boldsymbol{\hat{\theta}}_{t|t-1}^{(k)}$, and $\hat{y}_{k,t}=\boldsymbol{\hat{\theta}}_{t|t-1}^{(k)T}\boldsymbol{x}_{t}$
\State  derive loss $\ell_{k,t}$
\EndFor
\State update:
$\boldsymbol{p}_{t+1}(k) =  \frac{\boldsymbol{p}_t(k)e^{-\eta \ell_{k,t}}}{\sum_{j=1}^K \boldsymbol{p}_t(j)e^{-\eta \ell_{j,t}}}$
\State update: $\boldsymbol{p}_{t+1} = M\boldsymbol{p}_{t+1}$ 
\State Select $Q_{t} =Q^{(k)}$  with proba $p_{t+1}(k)$
\State Compute $P_{t+1|t}, \boldsymbol{\hat{\theta}}_{t+1|t}$ 
\State Forecast $\hat{y}_{t+1}=\boldsymbol{\hat\theta}_{t+1|t}^T\boldsymbol{x}_{t+1}$ 
\EndFor
\end{algorithmic}
\end{algorithm} 
The main parameters of this algorithm are $\eta$ and $M$. To reduce the complexity of the problem we will assume in our experience that $M$ only depends of a parameter $\alpha$ such that :
$$M_\alpha = \begin{pmatrix}
1-\alpha & \frac{\alpha}{K-1} & \dots & \frac{\alpha}{K-1}\\
\frac{\alpha}{K-1} & \ddots & \ddots & \vdots \\
\vdots& \ddots &\ddots &\frac{\alpha}{K-1} \\
\frac{\alpha}{K-1}& \dots & \frac{\alpha}{K-1} & 1-\alpha
\end{pmatrix}$$
We can then calibrate $\eta$ and $\alpha$ by grid-search based on the best cumulative loss.
\begin{prop}
\label{prop::KFMHrand}
Algorithm KFMH randomized, with transition matrix $M_\alpha$ and fined tuning for $\eta$ and $\alpha$, satisfies: 
   $$ \E[R_T^*] \le {\cal O}(\sqrt{S T \log T})\,,\qquad T\ge 1,$$
where $S$ is the number of regime switching of the SKF competitor.
\end{prop}
\begin{proof}
Without loss of generality we consider that the random selection at step $t$ following the distribution ${\bf p}_t$ is independent of ${\cal F}_t$ but included in ${\cal F}_{t+1}$.
The tower property yields the identity
$$
\E\Big[\sum_{t=1}^T \ell_{t}\Big]=\E\Big[\sum_{t=1}^T \E[\ell_{t}\mid{\cal F}_{t}]\Big]=\E\Big[\sum_{t=1}^T \sum_{k=1}^Kp_t(k)\ell_{t,k}\Big]\,.
$$
Rewriting the regret bound of Proposition \ref{prop::FS} for the linearized loss and for the specific transition matrix $M_\alpha$, we obtain
$$
\E[R_T]\le  \dfrac{S+1}\eta \log\dfrac K{\alpha}+\dfrac{T-S}\eta \log\Big(\dfrac 1{1-\alpha}\Big)+\dfrac{\eta T}8 \,.
$$
The desired result follows after an optimisation of the regret bound in $\alpha$ and $\eta$.
\end{proof}

\subsection{Misspecified setting}

In well-specified cases, the SKF competitor is optimal and has small cumulative quadratic losses, see e.g. \cite{adjakossa2023kalman}. However, in real-world applications one needs to be more robust to specification. We introduce a new setting where the prediction is provided by a Kalman filter acting on a state equation driven by the covariance matrix
$$
\hat{Q}_t=\sum_{k=1}^K\boldsymbol{p}_{t}(k)Q^{(k)}\,.
$$
The interpretation is to transfer the confidence on each experts computed by $\boldsymbol{p}_{t}$ directly on the state Equation \ref{SKM::var::eq_state}. The gain is to design an estimation of the state variance and a forecaster that is a KF and not an aggregation. One also has to design each expert $k$ in a different way than pure independent KF. Indeed, while in Algorithm \ref{alg:KFMH_random} each expert was computing his prediction independently from the learner; now the expert starts from a prior knowledge $P_{t-2|t-2}, \boldsymbol{\hat{\theta}}_{t-1|t-2}$ of learner's previous predictions 
and derives new KF updates as follows:
\begin{align*}
    &P_{t-1|t-2}^{(k)}=P_{t-2|t-2}+Q^{(k)},\\
     &P_{t-1|t-1}^{(k)}=P_{t-1|t-2}^{(k)} - \frac{P_{t-1|t-2}^{(k)}\boldsymbol{x}_{t-1}\boldsymbol{x}_{t-1}^TP_{t-1|t-2}^{(k)}}{\boldsymbol{x}_{t-1}^TP_{t-1|t-2}^{(k)}\boldsymbol{x}_{t-1}+1},\\
     &\boldsymbol{\hat{\theta}}_{t-1|t-1}^{(k)}=\boldsymbol{\hat{\theta}}_{t-1|t-2}-P_{t-1|t-1}^{(k)}\boldsymbol{x}_{t}(\boldsymbol{\hat{\theta}}_{t-1|t-2}^T\boldsymbol{x}_{t}-y_{t}),  \\
     &\boldsymbol{\hat{\theta}}_{t|t-1}^{(k)}=\boldsymbol{\hat{\theta}}_{t-1|t-1}^{(k)},\end{align*}
to predict:
$$\hat{y}_{k,t} =\boldsymbol{\hat{\theta}}_{t|t-1}^{T(k)}\boldsymbol{x}_{t}.$$ 
Then the learner observes the loss $\ell_{k,t}$ of each expert, updates the weights and computes $\boldsymbol{\hat \theta}_{t+1|t},P_{t+1|t}$ applying KF updates of Theorem \ref{Kalman_update} with $\hat{Q}_{t+1}$.

\subsection{Choice of loss and optimization over $\sigma$}
In Section \ref{sec::Relatew_work} we considered the square loss. 
However, standard methods optimizing time-invariant variances in state-space models optimize the likelihood as the natural objective function \cite{Dur, davis2016introduction}. This leads us to consider the log-likelihood as loss function in our algorithm. Using the chain rule, the log-likelihood writes:
\begin{align*}
    & \log p(y_{1:T},\boldsymbol{x}_{1:T}|Q_{1:T-1},\sigma_{1:T}^2) \\
    & \qquad = \sum_{t=1}^T \log p(y_t|\boldsymbol{x}_t,y_{1:t-1},\boldsymbol{x}_{1:t-1},Q_{1:(t-1)},\sigma_{1:t}^2) + c\\
    & \qquad = \sum_{t=1}^T \ell_t(\boldsymbol{\hat{\theta}}_{t|t-1},P_{t|t-1},\sigma_t^2,y_t) + c \,.
\end{align*} where $c$ is a constant term and the loss function writes:
\begin{align*}
& \ell_t(\boldsymbol{\hat{\theta}}_{t|t-1},P_{t|t-1},\sigma_t^2,y_t)\\
& \quad = \log \N(y_t|\boldsymbol{\hat{\theta}}_{t|t-1}^T\boldsymbol{x}_t, \sigma_t^2 + \boldsymbol{x}_t^TP_{t|t-1}\boldsymbol{x}_t) \,.
\end{align*}
The dependence of this loss to $Q_{t-1}$ is hidden in the estimated state distribution represented by $\boldsymbol{\hat\theta}_{t\mid t-1}, P_{t\mid t-1}$.
The peculiarity of the log-likelihood is that it strongly depends of $\sigma_t^2$. This means that the assumption $\sigma_t^2=1$ is too restrictive. In our experiments we observed that it did not perform well to apply the expert framework to both variances. We believe that it's due to the fact that $Q_{t-1}$ and $\sigma_t^2$ are strongly related. Furthermore, estimating observation noise variance is an easier task compared to the estimation of the covariance matrix of a latent process. Therefore we directly apply a gradient descent algorithm, namely ADAM \cite{kingma2014adam}. Equation $(*)$ in Algorithm \ref{KFMHL} and function $f(.)$ refers to this ADAM gradient descent which involves 4 hyper-parameters (namely $\alpha_1, \beta_1, \beta_2, \epsilon)$ and the gradient of the loss over $\sigma$ which writes :
$$\frac{\partial\ell_t}{\partial \sigma} = -\frac{\sigma(y-\boldsymbol{\hat{\theta}}_{t|t-1}^T\boldsymbol{x}_t)^2}{\sigma^2+\boldsymbol{x}_t^TP_{t|t-1}\boldsymbol{x}_t}+\frac{\sigma}{\sigma^2+\boldsymbol{x}_t^TP_{t|t-1}\boldsymbol{x}_t}.$$
Note that by decoupling the estimation of both variances, we don't need {\it a priori} on the possible values of $\sigma_t^2$ and their relation to the possible values of $Q_t$. 

\subsection{Sliding Window}
We propose another improvement of Algorithm \ref{KFMHL} motivated by \cite{huang2020slide}. The main intuition behind this is to better discriminate between the different experts. We note that $P_{t-1\mid t-1}$ depends on $Q_1,Q_2,\hdots Q_{t-2}$. Our initial algorithm proposes that expert $k$ depends on $Q_1,Q_2,\hdots Q_{t-3},Q^{(k)}$, that is, we replace only $Q_{t-2}$. Instead, we propose to replace the last $\tau$ covariance matrices by $Q^{(k)}$ for each expert $k$.
This replacement of a window $\tau$ is also motivated by the fact we assume $\alpha \ll 1$ in our experiments (typically $\alpha=0.01$), that is, the covariance matrix is not switching too frequently from one mode to another.
Considering this sliding window helps capturing the dynamics by discriminating better between the experts.

Formally, instead of starting from $\boldsymbol{\hat{\theta}}_{t-2|t-2},P_{t-2|t-2}$, we start from $\boldsymbol{\hat{\theta}}_{t-\tau-1|t-\tau-1},P_{t-\tau-1|t-\tau-1}$. For each expert $k$, we derive Kalman updates using $Q_s=Q^{(k)}$ for $s \ge t-\tau-1$, yielding $\boldsymbol{\hat{\theta}}_{t|t-1}^{(k)}, P_{t\mid t-1}^{(k)}$. Consequently, this sliding window is integrated in Algorithm \ref{KFMHL} as the introduction of a new loss function $\ell^{(\tau)}_{k,t}=\N(y_t|\boldsymbol{\hat{\theta}}_{t|t-1}^{(k)T}\boldsymbol{x}_t, \sigma_t^2 + \boldsymbol{x}_t^TP_{t|t-1}^{(k)}\boldsymbol{x}_t)$ where the dependency in $\tau$ refers to the way we create $\boldsymbol{\hat{\theta}}_{t|t-1}^{(k)},P_{t|t-1}^{(k)}.$

\begin{algorithm}[!h]
\caption{KFMH}\label{KFMHL}
\begin{algorithmic}
\State \textbf{Require:} $\boldsymbol{p}_1 = K^{-1}\mathbf{1}_K,  \boldsymbol{M} =(M_{i,j})_{i,j=1}^{  K}, \sigma_0$
\For{$t\in [T]$}
\State Reveal $y_t$ and ${\boldsymbol x}_{t-1}$.
\For{$k \in [K]$}
\State compute $P_{t|t-1}^{(k)}$, $\boldsymbol{\hat{\theta}}_{t|t-1}^{(k)}$, and $\hat{y}_{k,t}=\boldsymbol{\hat{\theta}}_{t|t-1}^{(k)T}\boldsymbol{x}_{t}$
\State  derive loss $\ell_{k,t}^{(\tau)}$
\EndFor
\State for $k \in [K]$ update:
$\boldsymbol{p}_{t+1}(k) =  \frac{\boldsymbol{p}_t(k)e^{-\eta \ell_{k,t}}}{\sum_{j=1}^K \boldsymbol{p}_t(j)e^{-\eta \ell_{j,t}}}$
\State update: $\boldsymbol{p}_{t+1} = M\boldsymbol{p}_{t+1}$ 
\State Select $Q_{t} =\sum_{k=1}^K p_t(k)Q^{(k)}$ 
\State Compute $P_{t+1|t}, \boldsymbol{\hat{\theta}}_{t+1|t}$ 
\State Forecast $\hat{y}_{t+1}=\boldsymbol{\hat\theta}_{t+1|t}^T\boldsymbol{x}_{t+1}$ 
\State Observe loss $\ell_t(\boldsymbol{\hat{\theta}}_{t|t-1},P_{t|t-1},\sigma_t^2,y_t)$
\State Update $\sigma_t = \sigma_{t-1} + f\left(\frac{\partial \ell_t(\boldsymbol{\hat\theta}_{t|t-1},P_{t|t-1},y_t)}{\partial \sigma}\right)$ (*)
\EndFor
\end{algorithmic}
\end{algorithm} 

\section{Experiments}\label{sec:experiments}
We compare the different methods introduced on various data sets. We begin with synthetic data; first we generate the data following the state-space model introduced in Section \ref{sec::Relatew_work}, we call that well-specified and it is the most favourable setting; then, we generate data not following the state-space assumption anymore, we call that experiment misspecified and it is an intermediate step before the application to real data. We conclude by a real data set motivating our setting: electricity load forecasting. We show that our method outperforms the standard Kalman aggregation as long as the experts are well-known even in the misspecified context.

\subsection{Experiments on Synthetic Data}
\textit{In the following we consider $K=2$.}
\subsubsection{Well-Specified (WS) Data with $\sigma_t^2 = \sigma^2$}
\label{subsec::WS}
\paragraph{Data Generation}
~~\\
We consider the model given by Equations $(\ref{SKM::var::eq_obs})$,(\ref{SKM::var::eq_state}) and (\ref{var::eq_markov}) and generate data as follow :
\begin{align*}
    \mathbb{P}(Z_{t+1}=Z_t) &= 1-\alpha, \\
\boldsymbol{\theta}_{t+1} - \boldsymbol{\theta}_t &\sim \N(\boldsymbol{0}_{\R^3},Q_{Z_t}), \\
y_t -\boldsymbol{\theta}_t^T \boldsymbol{x}_t &\sim \N(0,1),
\end{align*}
with $Q_{Z_t} \in \{Q^{(1)},Q^{(2)}\}$ where: 
\begin{align*}
    Q^{(1)}&=\Delta(10^{-4},10^{-1},10^{-2}),\\
    Q^{(2)} &= \Delta(10^{-1},10^{-4},10^{-2}).
\end{align*} 
We distinguish three ways of generating $\boldsymbol{x}_t$:
i.i.d Gaussian design: $\boldsymbol{x}_t \sim \N(\boldsymbol{0}_{\R^d},\boldsymbol{1}_{\R^d}),$
i.i.d uniform design: $\boldsymbol{x}_t \sim \mathcal{U}(\boldsymbol{0}_{\R^d},\boldsymbol{1}_{\R^d})$ and non-i.i.d design: $\boldsymbol{x}_1$ is generated as before. Then for $j \in [\![1,3]\!]$ and $t \ge 2$, we consider $\boldsymbol{a}_{t,j} = \boldsymbol{x}_{t-1,j} + \epsilon_{t,j}$ where $\epsilon_{t,j} \sim \N(0_{\R^3},10^{-3})$ and we use $\boldsymbol{x}_{t,j} = \boldsymbol{a}_{t,j}$ if $0 \le \boldsymbol{a}_{t,j} \le 1$, otherwise $\boldsymbol{x}_{t,j} = \lceil \boldsymbol{a}_{t,j}\rceil - \boldsymbol{a}_{t,j}$.
We generate sequence of length $T = 1000$ with a transition parameters $\alpha=10^{-2}$. 

\textbf{Algorithm Parameterization}\\
For both Algorithm \ref{alg_MH} and \ref{KFMHL} we set $\eta = 1$, and take latence $\tau = 5$ for KFMH and initial $\sigma_0 = 0.8$. For ADAM optimizer $f(.)$ we take the step-size $\alpha_1 = 0.01$ and the others standard hyper-parameters as in \cite{kingma2014adam} $(\beta_1 = 0.9, \beta_2=0.999, \epsilon=10^{-8}).$  
\subsubsection{Misspecified (MS) Data}
\textit{In this subsection we forget the Markov representation and consider that we play against a sequence of expert which is generated by some specific function}. 
\paragraph{Data generation : Sinusoidal}
~~\\
As discussed before, we assume here that $\sigma_t^2$ and $Q_t$ are independent. We generate $Q_t$ and $\sigma_t^2$ in the following way :
\begin{equation}
\label{eq::ms:Q}
     Q_t = \cos^2\left(\frac{3\pi t}{T}\right)Q^{(1)} + \sin^2\left(\frac{3\pi t}{T}\right) Q^{(2)},
\end{equation}
\begin{equation}
\label{eq::ms:sigma}
     \sigma_t = \sigma^{(1)}\cos^2\left(\frac{5\pi t}{T}\right)  + \sigma^{(2)} \sin^2\left(\frac{5\pi t}{T}\right),
\end{equation}
where $Q^{(1)},Q^{(2)}$ are as in Subsection \ref{subsec::WS} and $\sigma^{(1)}=1, \sigma^{(2)}=2$.

\subsection{Real Data}
Lastly we present our results on Real Data. We apply our methodology to the case study of \cite{obst2021adaptive}. We forecast the half-hourly French national electricity consumption published by {\it Réseau de Transport d'Électricité} \footnote{\url{https://opendata.reseaux-energies.fr/}}, as well as meteorological data from {\it MétéoFrance}\footnote{\url{https://donneespubliques.meteofrance.fr/}}. We consider the adaptation of the following generalized additive model, built independently for the 48 half-hours of the day:
\begin{align}
    \label{eq:gam}
    & y_t = f(Toy_t) + \sum\limits_{i=1}^7 \alpha_i \mathbf{1}_{DayType_t =i} \\
    \nonumber
    & \quad + f(TempS95_t) + f(TempS99_t) + f(t, Temp_t) \\
    \nonumber
    & \quad + s(CPWind_t) + f(t) + \beta Lag1_t + \gamma Lag7_t + \varepsilon_t
\end{align}
where at each time $t$,
$y_t$ is the consumption,
$DayType_t$ is the day of the week between $1$ and $7$,
$Temp_t$ is the temperature,
$TempS95_t$ and $TempS99_t$ are exponential smoothings of the temperature of exponential parameters $0.95$ and $0.99$ respectively,
$CPWind_t$ is the cooling power of wind (c.f. \cite{ludwig2023probabilistic}, Section 2.2),
$Lag1_t$ and $Lag7_t$ are the delayed consumption (one day ago and seven days ago, respectively).

We follow the framework of \cite{obst2021adaptive}: we consider the adaptation of a linear combination of the frozen effects. Formally, Equation \eqref{eq:gam} yields a GAM of 9 effects, and we define $\boldsymbol{x}_t\in\mathbb{R}^9$ the vector composed of these linear and non-linear effects. Our state-space paradigm is integrated naturally as
\begin{align*}
    & y_t = \boldsymbol{\theta}_ t^T\boldsymbol{x}_t + \epsilon_t, \qquad \epsilon_t \sim \mathcal{N}(0,\sigma_{t}^2) \\
    & \boldsymbol{\theta}_{t+1} = \boldsymbol{\theta}_t + \nu_t, \qquad \nu_t \sim \mathcal{N}(0,Q_{t})
\end{align*}

Our data ranges from January 2012 to December 2022. It has been extensively studied that the coronavirus crisis has implied a brutal drop in the electricity consumption in many countries, and France in particular \cite{IEA2020, farrokhabadi2022day}. State-space models have provided a very efficient framework to tackle the poor performances of traditional forecasting methods during this period \cite{obst2021adaptive, de2022state}. We propose to apply a switching strategy, in order to win in stable as well as unstable periods. Intuitively, defining a {\it slow} covariance matrix $Q^{(slow)}$ and a {\it fast} covariance matrix $Q^{(fast)}$, a switching Kalman-based algorithm is able to use $Q^{(slow)}$ in {\it slow} mode and $Q^{(fast)}$ in {\it fast} mode, yielding improvements in both worlds.

To define $Q^{(slow)}$ and $Q^{(fast)}$, we rely on an optimization of constant variances introduced in \cite{obst2021adaptive}, Section 2.A.2. However, we optimize the variances on different time periods. Optimizing on pre-covid data (2012-2019) yields the {\it slow} mode, and optimizing on covid times (the year 2020) yields the {\it fast} mode. We apply the same optimization method to define  $\sigma^{2(slow)}$ and $\sigma^{2(fast)}$ associated. We compute our test performances on the period ranging from January 2021 to December 2022.
\subsection{Results}

\begin{figure}[!h]
{\includegraphics[width = 3in]{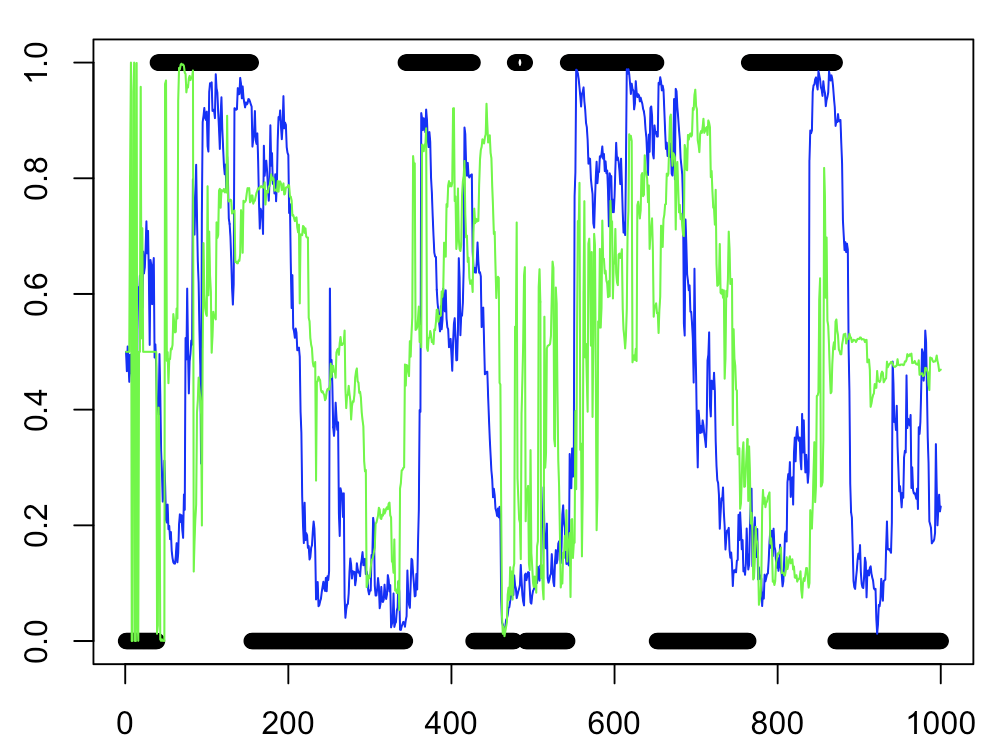}} 
{\includegraphics[width = 3in]{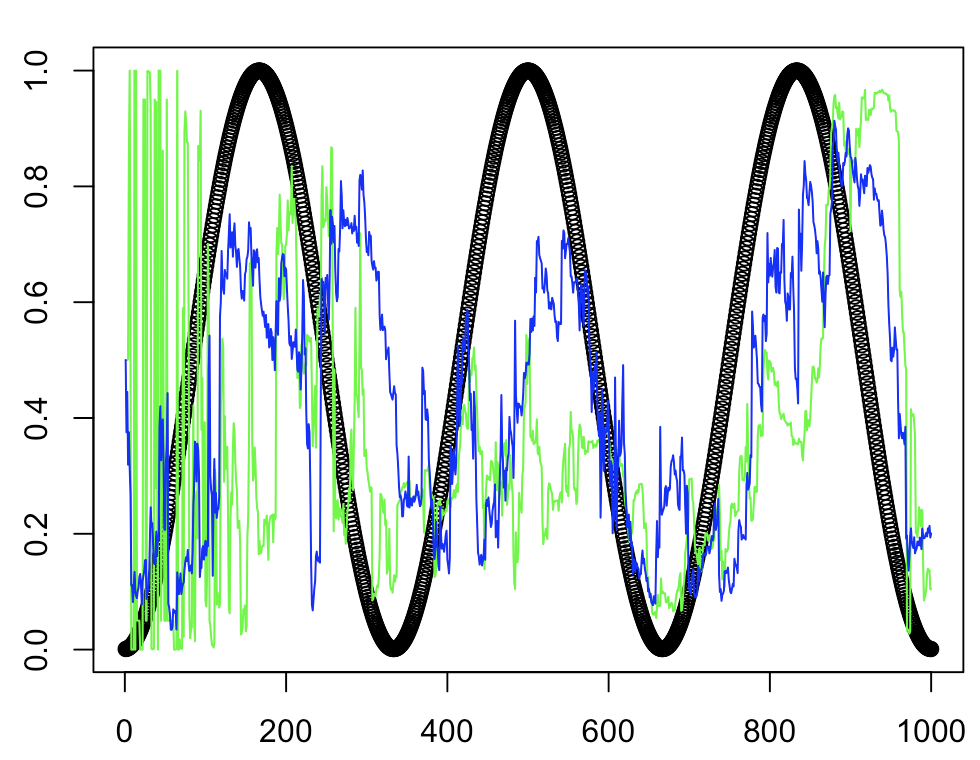}}\\
\caption{Evolution of the weights $\boldsymbol{p}_t$ for WS (left) and MS Data (right). }
\end{figure}
\begin{figure}
  {\includegraphics[width = 3in]{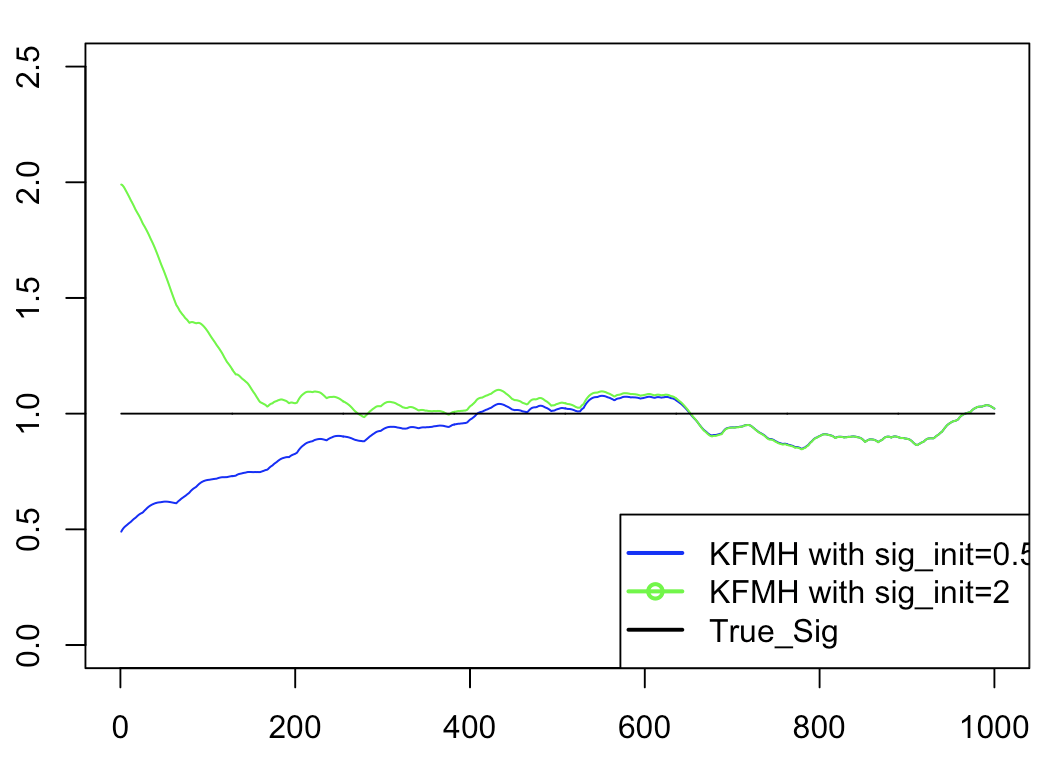}}
{\includegraphics[width = 3in]{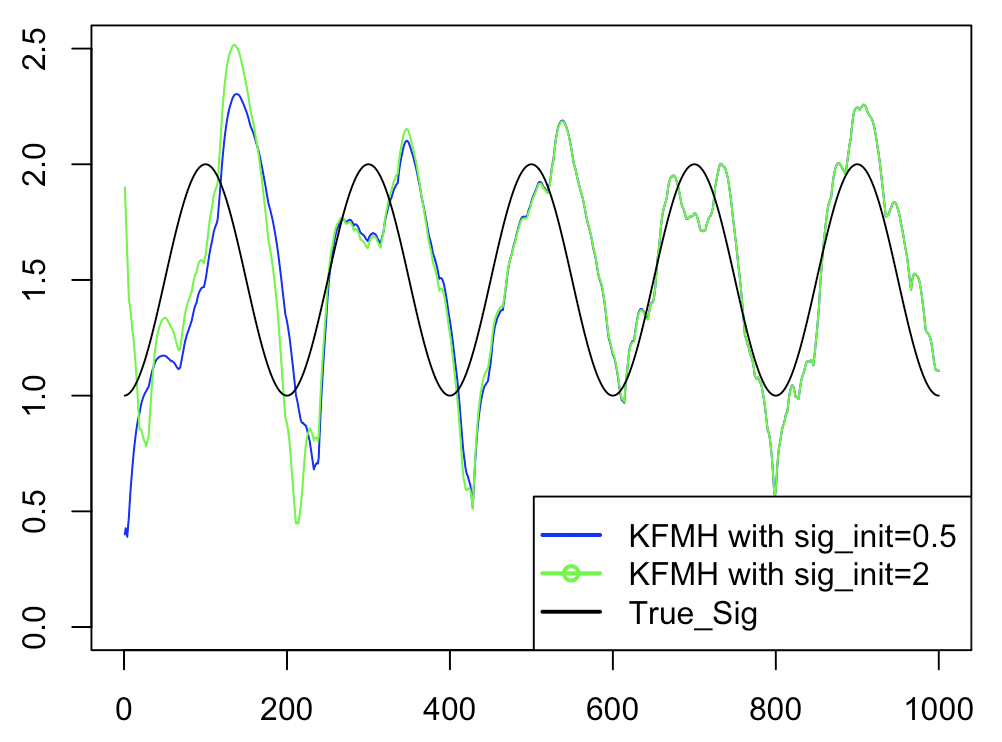}} 
\caption{Evolution of $\sigma_t$ starting from 2 different initial points for WS (left) and MS Data (right). }
\end{figure}

The results of our experiments on synthetic data are gathered in table 1 (see Supplementary materials). 
We generate $n_{iter}=1000$ sequences of length $T$ and compare our algorithm KFMH to differents methods using the Mean Square Error (MSE). First we consider the  {\it KF-adaptive} which is the optimal Kalman filter knowing $Q_t, \sigma_t^2$; this is the oracle or best possible forecast. Then we observe the performance of experts $1$ and $2$ which are Kalman filters with constant covariance matrix $Q_t=Q^{(1)}$ or $Q_t=Q^{(2)}$; we denote them by KF-$Q^{(1)}$ and KF-$Q^{(2)}$. We consider another Kalman filter with constant covariance matrix, KF-$Q^{(mean)}$, which applies $Q_t = \frac{Q^{(1)}+Q^{(2)}}{2}$. This is a better containder for our algorithm since in both WS and MS scenarios we generate data such that we approximately spend the same time in both regimes. Finally we present the results of KF-$agg$ the method that executes Algorithm \ref{alg_MH}. Except for our algorithm, every method has access to the true $\sigma_t^2$. We see that our algorithm outperforms the basic aggregation over the space of observation in every case. We observe that in the MS setting, KF-$Q^{(mean)}$ with knowledge of $\sigma_t^2$ beats our algorithm when generating i.i.d uniform and non-i.i.d $\boldsymbol{x}_t$, although  we are very close. Overall we observe that $Q^{(mean)}$ yields good performances on synthetic data; this was to expect because the covariance matrix of the latent process is $Q^{(1)}$ half of the time and $Q^{(2)}$ the other half. The results on real data are presented in table 2 (see Supplementary materials). For our evaluation, we consider the  first ten instants of the day over the $2021-2022$ period. Since we did not generate the data, {\it KF-adaptive} exists no longer and we only compare KFMH to the Kalman filters of constant variances (KF-$Q^{(slow)}$, KF-$Q^{(fast}$, KF-$Q^{(mean)}$), as well as the standard expert aggregation. This time we don't have access to the knowledge of $\sigma_t^2$ so each constant expert assumes $\sigma_t^2 = \sigma^2$ where $\sigma^2$ is respectively $\sigma^{2(slow)}$, $\sigma^{2(fast)}$ and $\sigma^{2(mean)}=\frac{\sigma^{2(slow)}+\sigma^{2(fast)}}{2}$. We find that we always outperform every constant expert, and beat the standard aggregation in average.

Figure \ref{fig:expe} shows the evolution of the weights $\boldsymbol{p}_t$ obtained by Algorithm \ref{KFMHL} (in blue) and Algorithm \ref{alg_MH} (in green) in well-specified as well as misspecified settings.
We also display the evolution of $\sigma_t$ estimated with Algorithm \ref{KFMHL} starting from two different initial points; respectively $\sigma_0=0.5$ (blue) and $\sigma_0 = 2$ (green) with respect to the true $\sigma_t$ (black) both in WS case ($\sigma_t^2=1)$ and MS context where $\sigma_t$ is generated with Equation \ref{eq::ms:sigma}. The figures suggest the consistency of the method and its robustness towards misspecification.

Below we provide the results of the experiments described in Section 4. We see that our algorithm is always better than a standard aggregation when performing on synthetic datasets, and always very close to the optimal. 
\begin{table}[h]
\caption{Comparison of MSE for each method on synthetic datas in both WS and MS case.} \label{tab::synthetic}
\begin{center}
\begin{tabular}{ccccccc}
\textbf{Method}  &KF-adpative & KF-$Q^{(1)}$ & KF-$Q^{(2)}$ & KF-$Q^{(mean)}$ & KF-$agg$ & KFMH\\
\hline \\
\textbf{i.i.d gauss WS}         &$1.802$ &$5.091$&$5.523$&$1.990$&$2.519$&$\textbf{1.916}$\\
\textbf{i.i.d unif WS}           &$1.796$& $4.021$ & $4.393$ &$1.931$ & $2.301$&$\textbf{1.891}$ \\
\textbf{non-i.i.d WS}            &$1.251$ &$1.471$&$1.415$&$1.267$&$1.301$&$\textbf{1.266}$ \\
\textbf{i.i.d gauss MS} &$1.908$&$5.192$&$5.526$&$1.988$&$2.799$&$\textbf{1.986}$ \\
\textbf{i.i.d unif MS}&$1.370$&$2.367$&$2.343$&$\textbf{1.393}$&$1.653$ &$1.402$\\
\textbf{non-i.i.d WS} &$2.758$&$3.040$&$2.964$&$\textbf{2.768}$&$2.823$&$2.777$
\end{tabular}
\end{center}
\end{table}

\begin{table}[h]
\caption{Comparison of RMSE over 2021-2022 period for each method on real data for 10 instants of the day.} \label{tab::synthetic}
\begin{center}
\begin{tabular}{ccccccc}
\textbf{Method}   &KF-$Q^{(slow)}$ & KF-$Q^{(fast)}$ & KF-$Q^{(mean)}$ & KF-$agg$ & KFMH\\
\hline \\
\textbf{Instant 1}         &$1005.572$ &$805.620$&$796.485$&$848.549$&$\textbf{788.284}$\\
\textbf{Instant 2}           &$888.951$& $808.081$ & $797.431$ &$790.475$ &$\textbf{786.788}$ \\
\textbf{Instant 3}           &$896.954$ &$838.322$&$827.544$&$\textbf{810.124}$&$812.846$ \\
\textbf{Instant 4} &$820.090$&$876.641$&$827.933$&$\textbf{812.816}$&$813.703$ \\
\textbf{Instant 5}&$842.068$&$848.340$&$844.974$&$849.140$&$\textbf{838.871}$\\
\textbf{Instant 6} &$842.665$&$846.738$&$845.461$&$842.816$&$\textbf{840.530}$ \\
\textbf{Instant 7}&$854.612$&$908.600$&$858.957$&$895.452$&$\textbf{849.253}$ \\
\textbf{Instant 8}&$853.396$&$926.782$&$860.292$&$\textbf{843.292}$&$847.000$ \\
\textbf{Instant 9} &$863.809$&$935.049$&$871.084$&$\textbf{855.107}$&$859.937$\\
\textbf{Instant 10} &$958.609$&$922.263$&$912.158$&$937.550$&$\textbf{901.185}$\\

\end{tabular}
\end{center}
\end{table}

\section{Conclusion}
In this paper, we have studied LGSSMs under unknown variances. We relied on online learning to estimate the variances. We considered the paradigm of SKF and used aggregation of experts to infer the variance of the process noise; to approximate the observation noise variance we applied an online gradient descent algorithm. Our experiments show good performances not only in the favourable well-specified setting but also under misspecification on both synthetic and real data, illustrating the robustness of our method. 

This paper opens different leads for future research. First, we specified on KF algorithms but we could apply our framework to other forecasting algorithms, involving the expert aggregation confidence weights at the right level, not necessarily on the predictions. Second, the choice of the diversity among the experts is crucial in aggregation \cite{gaillard2015forecasting}, and it is challenging to design different KF algorithms in an optimal way.

\bibliography{biblio}

\newpage

\appendix

\section{VB Derivation}
 \textit{In this section we provide the details of the derivation of the VB approach, mentioned in Section 2, that we found ineffective.}
\paragraph{}
~~\\
As mentioned VB 
approach looks for the pair $(q_\theta^*, q_Z^*)= (q_\theta^*(\boldsymbol{\theta}_{1:t}), q_Z^*(Z_{1:t}))$ that is the closest to the complete (\textit{i.e} taken over the whole sequence) joint posterior $p_{\F_t} = p(\boldsymbol{\theta}_{1:t},Z_{1:t}|\F_t)$ :
$$(q_\theta^*, q_Z^*) \in \text{argmin}_{(q_\theta, q_Z)\in \mathcal{D}} \ KL(q_\theta \times q_Z \ || \ p_{\F_t})$$

 One can show that :
\begin{equation}
      \ln q_\theta^*(\boldsymbol{\theta}_{1:t}) = \mathbb{E}_{q_Z(Z_{1:t})}\left[\ln p_{\F_t}\right] + c_1 \qquad  \qquad \qquad\text{(E-$\theta$ step)}
      \label{var::Etheta}
\end{equation}
\begin{equation}
    \ln q_Z^*(Z_{1:t}) =  \mathbb{E}_{q_\theta(\boldsymbol{\theta}_{1:t})}\left[\ln p_{\F_t}\right]  + c_2 \qquad \qquad \qquad \text{(E-$Z$ step)}
    \label{var::EZ}
\end{equation}

The aim is then to iteratively approximate the marginal  $q_\theta^*$ and $q_Z^*$, assuming the other is known and fixed. 
We provide further the closed form of what we call respectively E-$\theta$ step and E-$Z$ step. 
\subsection{E-$\theta$ step}
\begin{prop}(E-$\theta$ step)
~~\\
 For some $t > 1$, for any $q_Z(Z_{1:t})$ if we define for $s \in [\![1,t]\!]$ :
     \begin{align*}
         \overline{\sigma}_s^{-1} &= \sqrt{\overset{K}{\underset{k=1}{\sum }}q_Z(Z_s=k) \frac{1}{\sigma_{k}^2}}\\
          \overline{Q}_s^{-1} &=  \overset{K}{\underset{k=1}{\sum }}q_Z(Z_s=k)Q_{k}^{-1}
     \end{align*}
     then  the distribution $q_\theta^*$ which minimizes $KL(q_\theta \times q_Z || \ p_{\F_t})$ for $q_Z$ fixed  writes:
     \begin{equation}
     \label{var::eq_ln_qtheta}
   \ln q_\theta^*(\boldsymbol{\theta}_{1:t})=  
  \sum_{s=1}^{t} \Bigg(\ln \N(y_s|\boldsymbol{\theta}_s^T\boldsymbol{x}_s,\overline{\sigma}_s^2)+\ln \N(\boldsymbol{\theta}_s|\boldsymbol{\theta}_{s-1},\overline{Q}_{s-1})\Bigg)+ c
\end{equation}
where we denote $\boldsymbol{\theta}_0 = \boldsymbol{\hat{\theta}}_1$ and $\overline{Q}_0 = P_1$.
\label{var::prop_qtheta}.
 \end{prop}
 \begin{proof}
     For \(s \in [1 : t]\) we have:
\begin{align*}
\E_{q_Z}\left[\ln \N (y_s | \theta^T_s x_s, \sigma^2_{Z_s})\right] &= -\frac{1}{2} \sum_{1\leq k_1, \ldots, k_t \leq K} q_Z(Z_{1:t} = k_1, \ldots, k_t) \left(\frac{\boldsymbol{\theta}_s^T \boldsymbol{x}_s \boldsymbol{x}_s^T \boldsymbol{\theta}_s}{\sigma_{k_s}^2} -2y_s \frac{ \theta^T_s \boldsymbol{x}_s}{\sigma_{k_s}^2}\right) + c \\
&= -\frac{1}{2} \sum_{k_s=1}^K q_Z(Z_s = k_s) \left(\frac{\boldsymbol{\theta}^T_s \boldsymbol{x}_s x_s^T \boldsymbol{\theta}_s}{\sigma_{k_s}^2} - \frac{2y_s \boldsymbol{\theta}^T_s \boldsymbol{x}_s}{\sigma_{k_s}^2}\right) + c \\ 
\E_{q_Z}\left[\ln \N (\boldsymbol{\theta}_s | \boldsymbol{\theta}_{s-1},Q_{Z_s})\right] &= -\frac{1}{2} \sum_{k_s=1}^K q_Z(Z_{s-1} = k_{s-1}) \left(
\boldsymbol{\theta}_{s}^TQ_{k_{s-1}}^{-1}\boldsymbol{\theta}_{s}-2\boldsymbol{\theta}_{s}^TQ_{k_{s-1}}^{-1}\boldsymbol{\theta}_{s-1}+\boldsymbol{\theta}_{s-1}^TQ_{k_{s-1}}^{-1}\boldsymbol{\theta}_{s-1}\right) + c  
\end{align*}
Thus by denoting $\bar{\sigma}^{-1}_s = \sqrt{\sum_{k=1}^K q_Z(Z_s = k)\frac{1}{\sigma_k^2}}$ and $\bar{Q}_s^{-1} = \sum_{k=1}^K q_Z(Z_s = k) Q^{-1}_k$, we therefore have:
$$
\ln q_\theta^*(\boldsymbol{\theta}_{1:t}) = \sum_{s=1}^t \ln \N(y_s | \boldsymbol{\theta}^T_s \boldsymbol{x}_s, \bar{\sigma}_s^2) + \ln \N(\boldsymbol{\theta}_s | \boldsymbol{\theta}_{s-1}, \bar{Q}_{s-1}) + c$$
 \end{proof}
     This last proposition shows that we are exactly in the Kalman setting where the covariance matrix at each time step is the inverse of the average of the inverse covariance matrix over the law of $Z$. Indeed if we define the latent space model as :
 \begin{equation}
     \text{State :} \qquad \boldsymbol{\theta}_{t+1} = \boldsymbol{\theta}_t + \boldsymbol{\nu}_t  \ , \qquad \boldsymbol{\nu}_t \sim \N(\boldsymbol{0},\overline{Q}_t)
 \end{equation}
 \begin{equation}
     \text{Space : } \qquad y_t = \boldsymbol{\theta}_t^T\boldsymbol{x}_t + \epsilon_t \ , \qquad \epsilon_t \sim \N(0,\overline{\sigma}_t)
 \end{equation}
The log-likelihoof of the model writes as in (\ref{var::eq_ln_qtheta}). 
\begin{prop}
\label{E_theta}
 For some $t>1$ and $s \in [\![1:t]\!]$, if we denote the $ q_\theta^F(\boldsymbol{\theta}_s)$ the filtering approximation of$
p(\boldsymbol{\theta}_s|\boldsymbol{x}_{1:s},y_{1:s})$
 then we have:
  \begin{equation}
      q_\theta^F(\boldsymbol{\theta}_s) =  \N(\boldsymbol{\theta}_s|\boldsymbol{\hat{\theta}}_{s|s},P_{s|s})
      \label{var::eq_filtering}
  \end{equation} 
  with the following updates :
\begin{align*}
   & P_{s|s-1} = P_{s-1|s-1} + \overline{Q}_{s-1}  \\    &\boldsymbol{\hat{\theta}}_{s|s-1} = \boldsymbol{\hat{\theta}}_{s-1|s-1} \\
    & P_{s|s}= P_{s|s-1} -\frac{P_{s|s-1}\boldsymbol{x}_s\boldsymbol{x}_s^T P_{s|s-1}}{\boldsymbol{x}_s^T P_{s|s-1}\boldsymbol{x}_s + \overline{\sigma}_s^{2}}     \\
    &\boldsymbol{\hat{\theta}}_{s|s} = \boldsymbol{\hat{\theta}}_{s|s-1} -\frac{P_{s|s}}{\overline{\sigma}_s^{2}}\left(\boldsymbol{x}_s(\boldsymbol{\hat{\theta}}_{s|s-1}^T\boldsymbol{x}_s - y_s)\right) 
\end{align*}
The proof is immediate considering  Theorem 1  and the previous remark.
 \end{prop}
\subsection{E-$Z$ step}
\begin{prop}(E-$Z$ step)
\label{E_Z step}
~~\\
For some $t > 1$, for any $q_\theta(\theta_{1:t})$,
the distribution $q_Z^*$ which minimizes $KL(q_\theta \times q_Z || \ p_{\F_t})$ for $q_\theta$ fixed writes:
\begin{align*}
    q_Z^*(Z_{1:t}) &\propto  \pi(Z_1)  \mathcal{N}(\boldsymbol{\mu}_1|\boldsymbol{\hat{\theta}}_1,P_1) 
    \mathcal{N}(y_1|\boldsymbol{\mu}_1^T\boldsymbol{x}_1,\sigma_{Z_1}^2)\exp\left(-\frac{1}{2}\text{Tr}\left(\left[P_1^{-1}+\frac{\boldsymbol{x}_1\boldsymbol{x}_1^T}{\sigma_{Z_1}^2}\right]V_1\right)   \right)\\
& \prod_{s=2}^t\Bigg[ \mathcal{N}(y_s|\boldsymbol{\mu}_s^T\boldsymbol{x}_s,\sigma_{Z_s}^2)     \mathcal{N}(\boldsymbol{\mu}_s|\boldsymbol{\mu}_{s-1},Q_{Z_{s-1}})\exp\left(- \frac{1}{2}\text{Tr}\left[Q_{Z_{s-1}}^{-1}(V_s+ V_{s-1}-2W_{s-1})\right]\right)M_{Z_s|Z_{s-1}} \Bigg]
\end{align*}
where, for $s \in [\![1,t]\!]$  :
    $\boldsymbol{\mu}_s = \mathbb{E}_{q_\theta(\boldsymbol{\theta}_{1:t})}[\boldsymbol{\theta}_s],
    V_s = \mathbb{E}_{q_\theta(\boldsymbol{\theta}_{1:t})}[\boldsymbol{\theta}_s\boldsymbol{\theta}_s^T] - \boldsymbol{\mu}_s\boldsymbol{\mu}_s^T$ and $
    W_s = \mathbb{E}_{q_\theta(\boldsymbol{\theta}_{1:t})}[\boldsymbol{\theta}_s\boldsymbol{\theta}_{s+1}^T] - \boldsymbol{\mu}_s\boldsymbol{\mu}_{s+1}^T$
which can be rewritten as:
\begin{equation}
    q_Z(Z_{1:t}) = \rho_1(Z_1) \prod_{s=2}^t \rho_s(Z_s)\gamma_{Z_s|Z_{s-1}}
    \label{var::eq_HMM}
\end{equation}
where for $s > 1$: 
\begin{align*}
    \rho_1(Z_1) &\propto \pi(Z_1) \mathcal{N}(\boldsymbol{\mu}_1|\boldsymbol{\hat{\theta}}_1,P_1) 
    \mathcal{N}(y_1|\boldsymbol{\mu}_1^T\boldsymbol{x}_1,\sigma_{Z_1}^2)   \exp\left(-\frac{1}{2}\text{Tr}\left(\left[P_1^{-1}+\frac{\boldsymbol{x}_1\boldsymbol{x}_1^T}{\sigma_{Z_1}^2}\right]V_1\right)   \right) \\
    \rho_s(Z_s) &\propto \mathcal{N}(y_s|\boldsymbol{\mu}_s^T\boldsymbol{x}_s,\sigma_{Z_s}^2)\exp\left(-\frac{1}{2}\text{Tr}\left(\frac{\boldsymbol{x}_s\boldsymbol{x}_s^T}{\sigma_{Z_s}^2}V_s\right)\right)\\
    \gamma_{Z_s|Z_{s-1}}&\propto \mathcal{N}(\boldsymbol{\mu}_s|\boldsymbol{\mu}_{s-1},Q_{Z_{s-1}}) \exp\left(- \frac{1}{2}\text{Tr}\left[Q_{Z_{s-1}}^{-1}(V_s+ V_{s-1}-2W_{s-1})\right]\right)\tau_{Z_s|Z_{s-1}}
\end{align*}
  \label{var::prop_qZ}  
\end{prop}
This last expression (\ref{var::eq_HMM}) is equivalent to a \textbf{Hidden Markov Model (HMM)}. We define the filtration $\mathcal{H}_s = \sigma(x_{1:s},y_{1:s},\mu_{1:s},V_{1:s},W_{1:s-1})$ and $p'$ such that $q_Z(Z_{1:t}) = p'(Z_{1:t}|H_t)$. Therefore, by computing the standard forward-backward algorithm for HMMs, one can easily obtain the forward $q_Z^F(Z_s) = p' (Z_s|\mathcal{H}_s) \approx p(Z_s|\F_s)$ and backward $q_Z^B(Z_s)=p'(Z_s|\mathcal{H}_{s+1:t}) \approx p(Z_s|\boldsymbol{x}_{s+1:t},y_{s+1:t})$ and obtain the posterior approximation $q_Z(Z_s) \approx p(Z_s|\F_t)$ :
\begin{equation}
    q_Z(Z_s) \propto q_Z^F(Z_s)q_Z^B(Z_s)
    \label{var::eq_fb}
\end{equation}
\begin{prop}
    The forward-backward recursions write for some $t>1$ and $s \in [\![1,t]\!]$ :
    \begin{equation}
        q_Z^F(Z_s) \propto \rho_s(Z_s)   \odot (q_Z^F(Z_{s-1}) \gamma_{Z_s|Z_{s-1}})
    \end{equation}
    \begin{equation}
        q_Z^B(Z_s) \propto \gamma_{Z_{s+1}|Z_{s}} \left(q_Z^B(Z_{s+1}) \odot \rho_{s+1}(Z_{s+1})\right)
    \end{equation}
    \label{var::prop_fb}
\end{prop}
\begin{proof}

Starting from the joint distribution we have :
\begin{align*}
    \alpha_s(Z_s)&= p(Z_s,x_{1:s},y_{1:s},\mu_{1:s},V_{1:s},W_{1:s-1}) \\
    &=\sum_{k = 1}^K p(Z_s,Z_{s-1}=k,y_{1:s},x_{1:s},\mu_{1:s},V_{1:s},W_{1:s-1}) \qquad \qquad \text{(marginalizing over $Z_{s-1}$)}\\
    &= \sum_{k = 1}^K p(y_s|Z_s,Z_{s-1}=k,y_{1:s-1},x_{1:s},\mu_{1:s},V_{1:s},W_{1:s-1}) p(Z_s,Z_{s-1}=k,y_{1:s-1},x_{1:s},\mu_{1:s},V_{1:s},W_{1:s-1})\\
    &= \sum_{k = 1}^Kp(y_s|Z_s,Z_{s-1}=k, x_{s},\mu_s,V_s,W_{s-1}) p(Z_s|Z_{s-1}=k) \\
    & p(x_s|Z_{s-1}=k,x_{1:s-1},y_{1:s-1},\mu_{1:s-1},V_{1:s-1},W_{1:s-2})
   p(Z_{s-1}=k,x_{1:s-1},y_{1:s-1},\mu_{1:s-1},V_{1:s-1},W_{1:s-2})\\
    &\propto \sum_{k=1}^K \rho_s(Z_s)\gamma_{Z_s|Z_{s-1}= k}\alpha_{s-1}(Z_{s-1}=k)\\
    &\propto  \rho_s(Z_s) \odot \left(\alpha_{s-1}(Z_{s-1}) \gamma_{Z_s|Z_{s-1}}\right) \\
\Longrightarrow    q_Z^F(Z_s) &= p(Z_s|x_{1:s},y_{1:s})\\
    &\propto \rho_s(Z_s)   \odot (q_Z^F(Z_{s-1}) \gamma_{Z_s|Z_{s-1}})\\
    \\
   q_Z^B(Z_s) &= p(y_{s+1:t},x_{s+1:t}|Z_s)\\ 
   &= \sum_{k=1}^K p(y_{s+1:t},x_{s+1:t},Z_{s+1}=k|Z_s) \qquad \qquad \text{(marginalizing over $Z_{s+1}$)}\\
   &= \sum_{k=1}^K p(y_{s+1}|x_{s+1:t},y_{s+2:t},Z_{s+1}=k, Z_s)p(x_{s+1:t},y_{s+2:t},Z_{s+1}=k, Z_s)\\
   &=\sum_{k=1}^K p(y_{s+1}|x_{s+1},Z_{s+1}=k, Z_s)p(Z_{s+1}=k|Z_s,x_{s+1:t},y_{s+2:t}) p(x_{s+1:t},y_{s+2:t},Z_{s+1}=k)\\
   &\propto \sum_{k=1}^K  \rho_{s+1}(Z_{s+1}=k)
  q_Z^B(Z_{s+1}=k) \gamma_{Z_{s+1}=k|Z_s}\qquad \qquad \text{(à $p(x_{s+1}|.)$ près)}\\
    &\propto \gamma_{Z_{s+1}|Z_{s}} (q_Z^B(Z_{s+1}) \odot \rho_{s+1}(Z_{s+1}))
\end{align*}
\end{proof}
\subsection{VB Algorithm}
Finally the VB algorithm can writes as follow. The parameter $n_{iter}$ corresponds to the number of iteration we need to converge and usually is of order $5$. 
\begin{algorithm}
\caption{VB}\label{alg:cap}
\begin{algorithmic}
\State \textbf{Require:}  $q_{Z}^{(0)}, (\boldsymbol{\mu_s^{(0)}})_{s=1}^t,(V_s^{(0)})_{s=1}^t,(W_s^{(0)})_{s=1}^{t-1} (\boldsymbol{x}_s)_{s=1}^t,(y_s)_{s=1}^t$ 
\For{$k\in [\![1,n_{iter}]\!]$}:
\State 
 For {$s \in [\![1,t]\!]$}, compute $\bar\sigma_s,\bar{Q}_{s-1}$. 
 \State For {$s \in [\![1,t]\!]$}, derive $q_\theta^F(\boldsymbol{\theta}_s)$ with Proposition \ref{E_theta}.\Comment{E-$\theta$ step}
\State For {$s \in [\![1,t]\!]$} compute $\boldsymbol{\mu}_s,V_s,W_s$. 
\State  For {$s \in [\![1,t]\!]$}, derive $q_Z(\boldsymbol{\theta}_s)\propto q_Z^F(Z_s)q_Z^B(Z_s)$ with Proposition \ref{E_Z step}.  \Comment{E-$Z$ step}
\EndFor
\end{algorithmic}
\end{algorithm}

\end{document}